\documentclass{article}

\usepackage{amsmath}
\usepackage{amssymb}
\usepackage{amsfonts}
\usepackage{graphics}
\usepackage{booktabs}
\usepackage{newalg}
\usepackage{latexsym}
\usepackage{amsthm}
\usepackage{nicefrac}
\usepackage{stmaryrd}
\usepackage{url}

\def\REM#1{}

\def\deg#1#2{{}^{#1\!\!}/#2}

\newcommand{\upts}{\uparrow}
\newcommand{\downts}{\downarrow}

\def\tu#1{\langle #1\rangle}
\newcommand{\up}{\uparrow}
\newcommand{\down}{\downarrow}

\newtheorem{theorem}{Theorem}%
\newtheorem{example}{Example}

\newcommand{\noticka}[1]{}

\def\deg#1#2{{}^{#1\!}/#2}

\begin{document}
%






%

\begin{center}
\textbf{\Large Discovery of factors in matrices with grades}

\bigskip

{\large Radim Belohlavek, Vilem Vychodil

\medskip

Data Analysis and Modeling Lab\\
Dept. Computer Science, Palack\'y University,  Czech Republic\\
   e-mail: radim.belohlavek@am.org, vychodil@acm.org

}  
\end{center}

\medskip

\begin{abstract}
  We present an approach to decomposition and factor analysis of matrices 
  with ordinal data.
  The matrix entries are grades to which objects represented by rows 
  satisfy attributes represented by columns, e.g. grades to which 
  an image is red, a product has a given feature, or a person performs 
  well in a test.
  We assume that the grades form a bounded scale 
 equipped with certain aggregation operators and  conforms
  to the structure of a complete residuated lattice.
  We present a greedy approximation algorithm for the problem of decomposition 
  of such matrix in a product of two matrices with grades
   under the restriction that the number of factors be small.
  Our algorithm is based on a geometric insight provided by a 
  theorem identifying
  particular rectangular-shaped submatrices as optimal factors for the
  decompositions. These factors correspond to formal concepts of the input data 
  and allow  an easy interpretation of the decomposition. 
  We present illustrative examples and experimental evaluation.

\bigskip

\noindent
\textbf{Keywords:}
factor analysis, ordinal data,
 fuzzy relation,  fuzzy logic, concept lattice
\end{abstract}




\section{Introduction}

\paragraph{Problem Description}

Data dimensionality reduction is 
fundamental  for understanding and management of data.
In traditional approaches, such as factor analysis, a decomposition of 
an object-variable matrix
is sought into an object-factor matrix and a factor-variable matrix with
the number of factors reasonably small.
Compared to the original variables, the factors are considered more fundamental 
concepts, which are hidden in the data.
Their discovery and interpretation, which is central importance in our paper,
helps better understand the data.

In this paper, we consider decompositions of matrices $I$ with 
a particular type of ordinal data. Namely, each entry $I_{ij}$ of $I$ 
represents a grade to which the
object corresponding to the $i$th row has, or is incident with, the attribute
corresponding to the $j$th row. 
Examples of such data are results of questionnaires where
respondents (rows) rate services, products, etc., according to various
criteria (columns); results of performance evaluation of people (rows) by 
various tests (columns); or binary data in which
case there are only two grades, $0$ (no, failure) and $1$ (yes, success).
Our goal is to
decompose an $n\times m$ object-attribute matrix $I$ into 
a product
\begin{equation}
\label{eqn:IAB}
    I=A\circ B
\end{equation}
of an $n\times k$ object-factor
matrix $A$ and a $k\times m$ factor-attribute matrix $B$ 
with the number $k$ of factors as small as possible.

The scenario is thus similar to ordinary matrix decomposition problems but there 
are important differences. 
First, we assume that the entries of $I$, i.e. the grades, as well as the entries
of $A$ and $B$ are taken from a bounded scale $L$ of grades, such as 
the real unit interval $L=[0,1]$ or the Likert scale $L=\{1,\dots,5\}$ of degrees of satisfaction.
Second, the matrix composition operation $\circ$ used in our decompositions
is not the usual matrix product. 
Instead, we use the t-norm-based product with a t-norm $\otimes$ being a function 
used for aggregation of grades.
In particular, $A\circ B$ is defined by
\begin{equation}
\label{eqn:dec}
   (A\circ B)_{ij}=\textstyle\bigvee_{l=1}^k A_{il}\otimes B_{lj}.
\end{equation}
where $\bigvee$ denoted the supremum (maximum, if $L$ is linearly ordered).
The ordinary Boolean matrix product is a particular case of
this product in which the scale $L$ has $0$ and $1$ as the only grades
and $a\otimes b=\min(a,b)$.
Also, when $A$ and $B$ are thought of as fuzzy relations, $A\circ B$
is exactly the usual composition of fuzzy relations, see e.g. \cite{Got:TMVL,KlYu:FSFL}.
It is to be emphasized that we attempt to treat graded incidence
data in a way which is compatible with its semantics.
This need has been recognized long ago in mathematical
psychology, in particular in measurement theory \cite{Krea:FM}.
For example, even if we represent the grades by numbers such as 
$0\sim\,$strongly disagree,
$\frac{1}{4}\sim\,$disagree, \dots, $1\sim\,$strongly agree,
addition, multiplication by real numbers, and linear combination of graded
incidence data may not have natural meaning. 
Consequently, decomposition
of a matrix $I$ with grades into the ordinary matrix product of arbitrary
real-valued matrices $A$ and $B$ suffers from a difficulty to interpret 
$A$ and $B$, as well as to interpret the way  $I$ is reconstructed from, or 
explained by, $A$ and $B$.
This is not to say that the usual matrix decompositions of incidence data $I$
may not be useful.  For example, \cite{MiMiGiDaMa:TDBP,TaMiGiMa:Widybd} report 
that decompositions of binary matrices into real-valued matrices may yield better 
reconstruction accuracies. Hence, as far as the dimensionality reduction aspect
(the technical aspect) is concerned, ordinary decompositions may be 
favorable. However, when the knowledge discovery aspect plays
a role, attention needs to be paid to the semantics of decomposition.
Our algorithm is based on \cite{Bel:Odmerl}, in particular on using formal concepts of $I$
as factors.  This is important both from the technical viewpoint,
since due to \cite{Bel:Odmerl} optimal decompositions may be obtained this way,
and the knowledge discovery viewpoint, since formal concepts may naturally
be interpreted.


\paragraph{Related Work}
\label{sec:rw}

Recently, new methods of matrix decomposition and dimensionality reduction
have been developed.  
One aim is to have methods which are capable of discovering possibly
non-linear relationships between the original space and the lower 
dimensional space \cite{RoSa:Ndrlle,Teea:Ggfndr}. 
Another is driven by the need to take into account constraints 
imposed by the semantics of the data. 
Examples include nonnegative matrix factorization, in which
the matrices are constrained to those with nonnegative entries
and which leads to additive parts-based discovery of features
in data \cite{LeSe:Lponmf}.
Another example, relevant to this paper,
is Boolean matrix decomposition.
Early work on this problem was done in
\cite{NaMaWoAm:Mahlas,Sto:Sbpinpc}
which already include complexity results showing the hardness of problems
related to Boolean matrix decompositions. 
Recent work on this topic includes
\cite{BeVy:Dof,FrHuMuPo:Bfahlam,MiMu:Bfa,MiMiGiDaMa:TDBP,NaMaWoAm:Mahlas}.
As was mentioned above, Boolean matrix decomposition is a particular case of the 
problem considered in this paper. In particular, the present approach is inspired by  \cite{BeVy:Dof}.

Note also that partly related to this paper are 
methods for decomposition of binary matrices into
non-binary ones
such as \cite{Lee:Pcabd,SaOr,Scea:Glmpca,TaTa:Bpca,ZiVe:Ticabi},
see also \cite{TaMiGiMa:Widybd} for further references.

\section{Decomposition and Factors}
\label{sec:df}

\subsection{Decomposition and the Factor Model}

As was mentioned above, we assume that for the problem
of finding a decomposition (\ref{eqn:IAB}) of $I$ with the
matrix product defined by (\ref{eqn:dec}), 
the set $L$ of grades forms a bounded scale equipped with an aggregation
operation $\otimes$.
In particular, we assume that $L$  is a complete lattice
bounded by $0$ and $1$ and that $\otimes$ is a binary operation on $L$
that is commutative, associative, has $1$ as its neutral element, and
commutes with suprema, i.e.
\[\textstyle
   a\otimes\bigvee_{k\in K} b_k = \bigvee_{k\in K} (a\otimes b_k).
\]
Note that this in particular implies $a \otimes 1 = a$.
It is well-known, see e.g. \cite{Got:TMVL,Haj:MFL,KlMePa:TN},
 that for any such operation, one may define its residuum
$\to$  by
\begin{equation}
  \nonumber
  a\rightarrow b =\max\{c\in L\,|\, a\otimes c\leq b\}.
\end{equation}
The residuum satisfies an important technical condition called adjointness,
namely,
\[
  a\otimes b\leq c \mbox{ if{}f } a\leq b\rightarrow c.
\]
$L$ together with $\otimes$ and $\rightarrow$ satisfying the above conditions 
forms a complete residuated lattice \cite{WaDi:Rl}.

Complete residuated lattices are well known in fuzzy logic where are used as
the structures of truth degrees with $\otimes$ and $\to$ being the truth
functions of (many-valued) conjunction and implication.
Important examples include those with $L=[0,1]$ and $\otimes$ being
a continuous t-norm, such as $a\otimes b=\min(a,b)$ (G\"odel t-norm),
$a\otimes b=a\cdot b$ (Goguen t-norm), and 
$a\otimes b=\max(0,a+b-1)$ ({\L}ukasiewicz t-norm);
or $L$ being a finite chain equipped with the restriction of G\"odel t-norm,
{\L}ukasiewicz t-norm, or other suitable operation.
Since these matters are routinely known, we omit details and refer the reader
for further examples and properties of residuated lattices
to \cite{Got:TMVL,Haj:MFL,KlMePa:TN}.

Consider now the meaning of the factor model given by 
(\ref{eqn:IAB}) and  (\ref{eqn:dec}).
The matrices $A$ and $B$ represent relationships between objects and factors,
and between factors and the original attributes. 
We interpret $A_{il}$ as the degree to which the factor $l$ applies to the object $i$,
i.e. the truth degree of the proposition 
``factor $l$ applies to object $i$'';
and $B_{lj}$ as the degree to which the attribute $j$ is a particular manifestation
of the factor $l$, i.e. the truth degree of the proposition 
``attribute $j$ is a manifestation of factor $l$''.
Therefore, due to basic principles of fuzzy logic,
if $I=A\circ B$, the discovered factors explain the original relationship
between objects and attributes, represented by $I$, via $A$ and $B$
as follows:
the degree $I_{ij}$ to which the object $i$ has the attribute $j$ equals
the degree of the proposition
``there exists factor $l$ such that $l$ applies to $i$ and $j$ is a particular manifestation
of $l$''. 

As the nature of the relationship between objects and attributes via factors
 is traditionally of interest, it is worth noting that 
in our case,
the attributes are expressed by means of factors in a non-linear manner:

\begin{example}\label{ex:nonl}
\upshape
With {\L}ukasiewicz t-norm, let $I=A\circ B$ be 
\[
  \renewcommand{\arraystretch}{0.8}
  \arraycolsep=0.7mm 
   \left(
    \begin{array}{ccc}
       0.3 &  0.0 & 0.1 \\
       0.3 &  0.7 & 0.5 \\
       0.5 &  0.8 & 0.6
    \end{array}
   \right)=
   \left(
    \begin{array}{cc}
       0.2 &  0.8 \\
       0.9 &  0.8  \\
       1.0 &  1.0 
    \end{array}
   \right)
   \circ
   \left(
    \begin{array}{ccc}
       0.4 &  0.8 & 0.6 \\
       0.5 &  0.2 & 0.3
    \end{array}
   \right)\!\!.
\]
Then for $Q_1=(0.6\ 0.2)$ and $Q_2=(0.4\ 0.3)$
we have
$(Q_1+Q_2)\circ B=(1.0\ 0.5)\circ B=(0.4\ 0.8\ 0.6)
\not=(0.0\ 0.6\ 0.2)=(0.0\ 0.4\ 0.2)+(0.0\ 0.2\ 0.0)=Q_1\circ B + Q_2\circ B$.
\end{example}

\subsection{Factors for Decomposition}

We now need to recall a result from \cite{Bel:Odmerl}
saying that optimal decompositions of $I$ may be attained
by using formal concepts of $I$ as factors.
Denote by $L^U$ the set of all fuzzy sets in a set $U$
with truth degrees from $L$, i.e.
the set of all mappings from $U$ to $L$, and put 
$X=\{1,\dots,n\}$ (objects) and $Y=\{1,\dots,m\}$ (attributes).

  A \emph{formal concept} of $I$ is any pair $\tu{C,D}$ of fuzzy sets
  $C\in L^X$ and $D\in L^Y$ for which
  $C^\upts=D$ and $D^\downts=C$ where 
the operators ${}^\upts\!:L^X\rightarrow L^Y$ and 
  ${}^\downts\!:L^Y\rightarrow L^X$ are
defined by
  \begin{align*}
    C^\upts(j) &= \textstyle \bigwedge_{i\in X} (C(i)\rightarrow I_{ij})
    \quad\text{and}\quad
    D^\downts(i) = \textstyle \bigwedge_{j\in Y} (D(j)\rightarrow I_{ij}).
  \end{align*}
Here, $\bigwedge$ is the infimum in $L$ (in our case, 
since $X$ and $Y$ are finite, infimum coincides with minimum 
if $L$ is linearly ordered).
The set 
$${\cal B}(X,Y,I)=\{\tu{C,D}\in L^X\times L^Y \mid 
      C^\up=D \mbox{ and } D^\down=C\}$$ 
of all formal concepts of $I$ is called the 
\emph{concept lattice} of $I$
and forms indeed a complete lattice when equipped with a natural subconcept-superconcept 
ordering, see \cite{Bel:Clofl} for details. 
Formal concepts are simple models of concepts in the sense of traditional,
Port-Royal logic. 
For a formal concept $\tu{C,D}$,
$C$ and $D$ are called the
extent and the intent of $\tu{C,D}$;
the degrees $C(i)$ and $D(j)$ are interpreted as the 
degrees to which
the concept applies to object $i$ and attribute $j$, respectively.
The graded setting takes into account
that most concepts used by humans are graded rather than clear-cut. 

For a set 
%
\(
   {\cal F}=\{ \tu{C_1,D_1},\dots,\tu{C_k,D_k} \}
\)
of formal concepts of $I$ with a fixed order given by the indices, denote by $A_{\cal F}$ 
and $B_{\cal F}$ the $n\times k$ and $k\times m$ matrices defined by
\[
   (A_{\cal F})_{il}= (C_l)(i)  \quad\mbox{and}\quad
   (B_{\cal F})_{lj}= (D_l)(j).
\]
That is,  the $l$th column of $A_{\cal F}$ consists of grades assigned to 
the objects by $C_l$ and the $l$th row of $B_{\cal F}$ consists of grades assigned 
to attributes by $D_l$.

If $I=A_{\cal F}\circ B_{\cal F}$, $\cal F$ can be seen as a set of
factors which fully explain the data. In such a case, we call
the formal concepts from $\cal F$ \textit{factor concepts}.
In this case, the factors have a natural, easy-to-understand meaning as 
is demonstrated in Section \ref{sec:ie}.
Let $\rho(I)$ denote the Schein rank of $I$, i.e. 
\[
    \rho(I) = \min\{ k \mid I = A\circ B \mbox{ for some } n\times k
     \mbox{ and } k\times m \mbox{ matrices } A \mbox{ and } B
    \}.
\]
The following theorem was proven in 
 \cite{Bel:Odmerl}. 

%

\begin{theorem}
\label{thm:fcaof}
  For every $n\times m$ matrix $I$ with entries from $L$
  there exists a set ${\cal F}\subseteq {\cal B}(X,Y,I)$ 
  containing exactly $\rho(I)$ formal concepts for which 
  \(
     I=A_{\cal F}\circ B_{\cal F}.     
  \)
\end{theorem}

The theorem says that, in a sense, formal concepts
of $I$ are optimal factors for decompositions.
It follows that when looking for decompositions of $I$, one can restrict 
the search to the set of formal concepts instead of the set of all
possible decompositions.

\section{Algorithm and Complexity of Decompositions}

To prevent misunderstanding, let us define our problem precisely.
For a given (that is, constant for the problem) structure of truth degrees,
i.e. set $L$ equipped with the lattice operations and $\otimes$ and $\to$,
the problem we discuss is a minimization (optimization) problem
\cite{Auea:CA} specified as follows:

\medskip
\noindent
\begin{tabular}{ll}
\textsc{Input:}  & $n\times m$ matrix $I$ with entries from $L$;\\
\textsc{Feasible Solution:} & $n\times k$ and $k\times m$ matrices $A$ and $B$ with entries\\
                                             & from $L$ for which $I=A\circ B$;\\
\textsc{Cost of Solution:} & $k$.
\end{tabular}

\medskip
As indicated above, due to Theorem \ref{thm:fcaof}, we look
for feasible solutions $A$ and $B$ in the form $A_\mathcal{F}$ and $B_\mathcal{F}$
for some $\mathcal{F}$.
Therefore, the algorithm we present in Section \ref{sec:a} computes a set 
$\mathcal{F}$ of formal concepts of $I$ for which $A_\mathcal{F}$ and $B_\mathcal{F}$
is a good feasible solution.
Our algorithm  runs in polynomial time but produces only suboptimal
solutions, i.e. $|\mathcal{F}|\geq \rho(I)$. 
As is shown in Section \ref{sec:cd}, this is a consequence of a fundamental limitation.
Namely, unless P=NP, there does not exist a polynomial
time algorithm producing optimal solutions to the decomposition problem.
We demonstrate experimentally in Section \ref{sec:ie}, however, that the quality of the solutions
provided by our algorithm is reasonable.

In this section as well as in Section \ref{sec:ie} we need the following 
``geometric'' insight.
Let us note that every formal concept $\tu{C_l,D_l}\in\mathcal{F}$ induces a matrix
$J_l=C_l\otimes D_l$ given by
\begin{align}
  (C_l\otimes D_l)_{ij} = C_l(i)\otimes D_l(j),
  \label{eqn:rec}
\end{align}
the rectangular matrix induced by $\tu{C_l,D_l}$ (it results by the Cartesian product
of $C_l$ and $D_l$).
Then $I=A_{\cal F}\circ B_{\cal F}$ means that 
\begin{equation}
    \label{eqn:supI}
    I_{ij} = (J_1)_{ij}\vee\cdots\vee (J_k)_{ij},
\end{equation}
i.e. $I$ is the $\bigvee$-superposition of $J_l$s.

\subsection{Algorithm}
\label{sec:a}

Throughout this section, we assume that $L$ is linearly ordered,
i.e. $a\leq b$ or $b\leq a$ for any two degrees $a,b\in L$.
(The general, non-linear case can be handled with no substantial difficulty
but we prefer to keep things simple, particularly because of the practical importance
of the linear case.)
In such case, (\ref{eqn:supI}) implies that $I=A_{\cal F}\circ B_{\cal F}$
if and only if for each $\tu{i,j}\in\{1,\dots,n\}\times\{1,\dots,m\}$
there exists $\tu{C_l,D_l}\in \mathcal{F}$ for which
\begin{equation}
\label{eqn:cover}
     I_{ij} = C_l(i) \otimes D_l(j).
\end{equation}
In case of (\ref{eqn:cover}), we say that $\tu{C_l,D_l}$ \emph{covers}
$\tu{i,j}$. 
This allows us to see that the problem of finding 
a set $\mathcal{F}$ of formal concepts of $I$ for which 
$I=A_{\cal F}\circ B_{\cal F}$ can be reformulated as the problem of
finding $\mathcal{F}$ such that every pair from the set
\begin{equation}
\label{eqn:U}
  \mathcal{U}=\{\langle i,j\rangle \,|\, I_{ij} \ne 0\}
\end{equation}
is covered by some
$\tu{C_l,D_l}\in \mathcal{F}$.
Since $C_l(i) \otimes D_l(j)\leq I_{ij}$ is always the case \cite{Bel:Fgc},
we need not worry about overcovering.
We now see that every instance of our decomposition problem may be rephrased as an 
instance of the well-known set cover problem, see e.g. \cite{Auea:CA,CoLeRiSt:IA} in which 
the set to be covered is
$\mathcal{U}$ and the system of sets that may be used to cover $\mathcal{U}$
is
\[
     \{  \{\tu{i,j}\,;\, I_{ij}\leq C(i)\otimes D(j)\} \mid \tu{C,D}\in\mathcal{B}(X,Y,I)\}.
\]
Accordingly, one can use the well-known greedy approximation algorithm \cite{Auea:CA} for solving
set cover to select a set $\mathcal{F}$ for formal concepts for which 
$I=A_{\cal F}\circ B_{\cal F}$. However, this would be a costly way
from the computational complexity point of view. Namely, one would need to compute
the possibly rather large set $\mathcal{B}(X,Y,I)$ first and, worse, repeatedly
iterate over this set in the greedy set cover algorithm.

Instead, we propose a different greedy algorithm. 
The idea is to supply promising candidate factor concepts \emph{on demand} during
the factorization procedure, as opposed to computing all candidate factor 
concepts beforehand.
The algorithm generates the promising candidate factor concepts by looking
for promising columns. 
A technical property which we utilize is the fact that 
for each formal concept $\langle C,D\rangle$,
\begin{align*}
  D = \textstyle\bigcup_{j \in Y}\{\deg{D(j)}{j}\}^{\downarrow\uparrow},
\end{align*}
i.e. each intent $D$ is a union of
intents $\{\deg{D(j)}{j}\}^{\downarrow\uparrow}$ \cite{Bel:Clofl} and
that $C=D^\downarrow$ by definition.
Here, $\{\deg{D(j)}{j}\}$ denotes a graded singleton, i.e.
\[
    \{\deg{D(j)}{j}\}(j') =
    \left\{
     \begin{array}{ll}
          D(j) & \text{if } j'=j,\\
          0 & \text{if } j'\not=j.
      \end{array}
     \right.
\]
As a consequence, we may construct any formal concept by adding
sequentially $\{\deg{a}{j}\}^{\downarrow\uparrow}$ to the empty 
set of attributes. Our algorithm follows a greedy approach
that makes us select $j \in Y$ and a degree $a \in L$ which maximize the size of
\begin{align}
  D \oplus_a j =
  \{\langle k,l\rangle \!\in \mathcal{U} \,|\,
  D^{+\downarrow}(k) \otimes 
  D^{+\downarrow\uparrow}(l) \geq I_{kl}\},
  \label{eqn:RV_set}
\end{align}
where $D^+ = D \cup \{\deg{a}{j}\}$ and 
$\cal U$ denotes the set of $\tu{i,j}$ of $I$ (row $i$, column $j$) for which
the corresponding entry $I_{ij}$ is not covered yet.
At the start, 
$\mathcal{U}$ is initialized according to  (\ref{eqn:U}).
As the algorithm proceeds, 
$\mathcal{U}$ gets updated by removing from it the pairs
$\tu{i,j}$ which have been covered by the selected formal concept
$\tu{C,D}$.
Note that $|D \oplus_a j|$ is the number of entries of $I$
which are covered by formal concept $\tu{D^{+\downarrow},D^{+\downarrow\up}}$,
i.e. by the concept generated by $D^+$, the intent of the current candidate concept
$\tu{C,D}$ extended by $\{\deg{a}{j}\}$.
Therefore, instead of going through all possible formal concepts and selecting
the factors from them,
we just go through columns and degrees and add them repeatedly as to maximize 
the value $V$ of the corresponding formal concepts, until such addition is possible.
The resulting algorithm is summarized below.

\medskip
\begin{algorithm}{Find-Factors}{I}
  \mathcal{U} \= \{\langle i,j\rangle \,|\, I_{ij} \ne 0\} \\
  \mathcal{F} \= \emptyset \\
  \begin{WHILE}{\mathcal{U} \ne \emptyset}
    D \= \emptyset \\
    V \= 0 \\
    \!\!\text{\textbf{select}} \langle j,a\rangle
    \text{\textbf{that maximizes}} |D \oplus_a j| \\
    \begin{WHILE}{|D \oplus_a j| > V}
      V \= |D \oplus_a j| \\
      D \= (D \cup \{\deg{a}{j}\})^{\downarrow\uparrow} \\
      \!\!\text{\textbf{select}} \langle j,a\rangle
      \text{\textbf{that maximizes}} |D \oplus_a j|
    \end{WHILE} \\
    C \= D^{\downarrow} \\
    \mathcal{F} \= \mathcal{F} \cup \{\langle C,D\rangle\} \\
    \begin{FOR}{\langle i,j\rangle \in \mathcal{U}}
      \begin{IF}{I_{ij} \leq C(i) \otimes D(j)} \\
        \mathcal{U} \= \mathcal{U} \mathop{\backslash} \{\langle i,j\rangle\}
      \end{IF}
    \end{FOR}
  \end{WHILE} \\
  \RETURN \mathcal{F}
\end{algorithm}

\medskip
The main loop of the algorithm (lines 3--16) is executed until all
the nonzero entries of $I$ are covered by at least one factor in $\mathcal{F}$.
The code between lines 4 and 10 constructs an intent by adding the most promising
columns. After such an intent $D$ is found, we construct the corresponding
factor concept and add it to $\mathcal{F}$. The loop between lines 13 and 16
ensures that all matrix entries covered by the last factor are removed
from $\mathcal{U}$. Obviously, the algorithm is sound and finishes after
finitely many steps
 (polynomial in terms of $n$ and $m$) with a set $\mathcal{F}$ of factor concepts.

\subsection{Complexity of Finding Optimal Decompositions}
\label{sec:cd}


As mentioned above, there is no guarantee that our algorithm finds an optimal
decomposition, i.e. the one with $k=\rho(I)$.
The following theorem shows that, unless P=NP, no polynomial time algorithm
which finds optimal decompositions exists.
 
\begin{theorem}
  The decomposition problem, i.e. the problem to find for a given $n\times m$ 
  matrix $I$ with grades an $n\times k$ matrix $A$  and a $k\times m$ matrix $B$ 
  for which $I=A\circ B$ with $k$ as small as possible,
  is NP-hard.
\end{theorem}
\begin{proof}
The theorem is an easy consequence of established reductions,
see \cite{Nau:Sc,NaMaWoAm:Mahlas} and also
\cite{BeVy:Dof,MiMiGiDaMa:TDBP,VaAtGu:Trmpfmdsf}.
Namely, by definition of NP-hardness of optimization problems, we need
to show that the corresponding decision problem is NP-complete.
The decision problem, which we denote by $\Pi$ in what follows, is to decide for a given 
$I$ and positive integer $k$ whether there exists
a decomposition $I=A\circ B$ with the inner dimension $k$ or smaller.
Now, $\Pi$ is NP-complete because the decision version of the set basis problem, which is
known to be NP-complete \cite{Sto:Sbpinpc}, is reducible to it.
The decision version of the set basis problem is:
Given a collection $S=\{S_1,\dots,S_n\}$ of sets $S_i\subseteq\{1,\dots,m\}$
and a positive integer $k$, is there a collection
$P=\{P_1,\dots,P_k\}$ of subsets $P_l\subseteq\{1,\dots,m\}$
such that for every $S_i$ there is a subset $Q_i\subseteq \{P_1,\dots,P_k\}$
for which $\bigcup Q_i=S_i$ (i.e., the union of all sets from $Q_i$ is equal
to $S_i$)? 
This problem is easily reducible to $\Pi$:
Given $S$, define an $n\times m$ matrix
$I$ by $I_{ij}=1$ if $j\in S_i$ and $I_{ij}=0$ if $j\not\in S_i$. 
Such reduction works for every $L$ and $\otimes$ because
we always have $1\otimes 1=1$ and $1\otimes 0=0\otimes 1=0\otimes 0=0$.
Namely,
one can check that if $I=A\circ B$ for $n\times k$ and $k\times m$
matrices $A$ and $B$ with entries from $L$ then 
$P_l$ ($l=1,\dots,k$) and $Q_i$, defined by $j\in P_l$ if $B_{lj}=1$
and $P_l\in Q_i$ if $A_{il}=1$, represent a solution to the set basis
problem given by $S$.
Conversely, if $P_l$ and $Q_l$ represent a solution to the set basis problem,
the matrices $A$ and $B$ defined by 
$B_{lj}=1$ if $j\in P_l$ and $B_{lj}=0$ if $j\not\in P_l$, and
$A_{il}=1$ if $P_l\in Q_i$ and $A_{il}=0$ if $P_l\not\in Q_i$,
are matrices with entries from $L$ which represent a solution to $\Pi$.
\end{proof}

\section{Examples and Experiments}
\label{sec:ie}

In Section \ref{sec:dd}, we examine in detail a factor analysis of 2004 Olympic
Decathlon data. We include this example to illustrate the notions involved
in our methods but most importantly to argue that the algorithm developed
in this paper can be used to obtain reasonable factors from data with grades.
In Section \ref{sec:ee}, we present results of an experimental evaluation
of our algorithm.

\subsection{Decathlon data}\label{sec:dd}



Grades of ordinal scales are conveniently represented by numbers,
such as the Likert scale $\{1,\dots,5\}$. In such a case we assume these
numbers are normalized and taken from the unit interval $[0,1]$.
As an example, the Likert scale is represented by 
$L=\{0,\frac{1}{4},\frac{1}{2},\frac{3}{4},1\}$.
Due to the well-known Miller's $7\pm2$ phenomenon \cite{Mil:Mns}, 
one might argue that
we should restrict ourselves to small scales.

In this section, we explore factors explaining the athletes' performance in the event.
Tab.\,\ref{tab:dec_input} (top) contains the
results of top five athletes in 2004 Olympic Games decathlon in points 
which are obtained using the IAAF Scoring Tables for 
Combined Events. Note that the IAAF Scoring Tables provide us with an ordinal
scale and a ranking function assigning the scale values to athletes.
We are going to look at whether this data can be explained using 
formal concepts as factors.

\begin{table}
  \caption{2004 Olympic Games decathlon}
  \label{tab:dec_input}
  \centering




  \medskip
  \small
  \textbf{Scores of Top 5 Athletes}

  \medskip
  \begin{tabular}{|@{\,}r@{\,}|*{10}{|@{\,}c@{\,}}|}
    \hline
    \rule{0pt}{9pt}
    & $10$
    & \emph{lj}
    & \emph{sp}
    & \emph{hj}
    & $40$
    & $11$
    & \emph{di}
    & \emph{pv}
    & \emph{ja}
    & \emph{$15$} \\
    \hline
    \hline
    \rule{0pt}{9pt}Sebrle &
    $894$ & $1020$ & $873$ & $915$ & $892$ &
    $968$ & $844$ & $910$ & $897$ & $680$ \\
    \hline
    \rule{0pt}{9pt}Clay &
    $989$ & $1050$ & $804$ & $859$ & $852$ &
    $958$ & $873$ & $880$ & $885$ & $668$ \\
    \hline
    \rule{0pt}{9pt}Karpov &
    $975$ & $1012$ & $847$ & $887$ & $968$ &
    $978$ & $905$ & $790$ & $671$ & $692$ \\
    \hline
    \rule{0pt}{9pt}Macey &
    $885$ & \hfill $927$ & $835$ & $944$ & $863$ &
    $903$ & $836$ & $731$ & $715$ & $775$ \\
    \hline
    \rule{0pt}{9pt}Warners &
    $947$ & \hfill $995$ & $758$ & $776$ & $911$ &
    $973$ & $741$ & $880$ & $669$ & $693$ \\
    \hline
  \end{tabular}

  \bigskip
  \textbf{Incidence Data Table with Graded Attributes}

  \medskip
  \small
  \begin{tabular}{|@{\,}r@{\,}|*{10}{|@{\,}c@{\,}}|}
    \hline
    \rule{0pt}{9pt}
    & $10$
    & \emph{lj}
    & \emph{sp}
    & \emph{hj}
    & $40$
    & $11$
    & \emph{di}
    & \emph{pv}
    & \emph{ja}
    & \emph{$15$} \\
    \hline
    \hline
    \rule{0pt}{9pt}Sebrle &
    $0.50$ & $1.00$ & $1.00$ & $1.00$ & $0.75$ &
    $1.00$ & $0.75$ & $0.75$ & $1.00$ & $0.75$ \\
    \hline
    \rule{0pt}{9pt}Clay &
    $1.00$ & $1.00$ & $0.75$ & $0.75$ & $0.50$ &
    $1.00$ & $0.75$ & $0.50$ & $1.00$ & $0.50$ \\
    \hline
    \rule{0pt}{9pt}Karpov &
    $1.00$ & $1.00$ & $0.75$ & $0.75$ & $1.00$ &
    $1.00$ & $1.00$ & $0.25$ & $0.25$ & $0.75$ \\
    \hline
    \rule{0pt}{9pt}Macey &
    $0.50$ & $0.50$ & $0.75$ & $1.00$ & $0.75$ &
    $0.50$ & $0.75$ & $0.25$ & $0.50$ & $1.00$ \\
    \hline
    \rule{0pt}{9pt}Warners &
    $0.75$ & $0.75$ & $0.50$ & $0.50$ & $0.75$ &
    $1.00$ & $0.25$ & $0.50$ & $0.25$ & $0.75$ \\
    \hline
  \end{tabular}

  \begin{flushleft}
    \textbf{Legend:}
    $10$---100 meters sprint race;
    $lj$---long jump;
    $sp$---shot put;
    $hj$---high jump;
    $40$---400 meters sprint race;
    $11$---110 meters hurdles;
    $di$---discus throw;
    $pv$---pole vault;
    $ja$---javelin throw;
    $15$---1500 meters run.
\end{flushleft}

\end{table}

We first transform the data from~\ref{tab:dec_input}\,(top) to
a five-element scale 
\begin{align}
  L = \{0.00,0.25,0.50,0.75,1.00\}
  \label{eqn:scale}
\end{align}
by a natural transformation and rounding. Namely, for each of
the disciplines, we first take the lower and highest scores achieved
among all athletes who have finished the 
decathlon event, see~Table~\ref{tab:scaling}.
Then, for each discipline, we make a linear transform of values from the
$[\min,\max]$ interval to the real unit interval. For instance,
in case of \emph{lj} (long jump), we consider the function
\begin{align}
  f_{lj}(x) &= \cfrac{x - 723}{(1050 - 723)} = 
  \cfrac{x - 723}{(1050 - 723)} = \cfrac{x - 723}{327}
  \label{eqn:transform}
\end{align}
ana analogously for the other disciplines, cf. Table~\ref{tab:scaling}.
Finally, for each athlete we compute the value of functions
like~\eqref{eqn:transform} and round the results to the closest value
from the discrete scale~\eqref{eqn:scale}. That is, instead of working
with numerical values as in Table~\ref{tab:dec_input}\,(top), we use
the graded dataset in Table~\ref{tab:dec_input}\,(bottom)
which describes the athletes' performance using the five-element scale
where the table entries are degrees to which athletes achieve high scores
for particular disciplines (with respect to the other athletes participating
in the event).
As a consequence, the factors then have a simple reading. 
Namely, the grades to which a factor applies to an athlete can be described
in natural language as ``not at all'', ``little bit'',
``half'', ``quite'', ``fully'', or the like. 

\begin{table}
  \caption{Lowest and highest scores in the 2004 Olympic Games decathlon}
  \label{tab:scaling}
  \centering
  \medskip
  \begin{tabular}{|@{\,}r@{\,}|*{10}{|@{\,}c@{\,}}|}
    \hline
    \rule{0pt}{9pt}
    & $10$
    & \emph{lj}
    & \emph{sp}
    & \emph{hj}
    & $40$
    & $11$
    & \emph{di}
    & \emph{pv}
    & \emph{ja}
    & \emph{$15$} \\
    \hline
    \hline
    \rule{0pt}{9pt}lowest &
    $782$ & \hfill $723$ & $672$ & $670$ & $673$ &
    $803$ & $661$ & \hfill $673$ & $598$ & $466$ \\
    \hline
    \rule{0pt}{9pt}highest &
    $989$ & $1050$ & $873$ & $944$ & $968$ &
    $978$ & $905$ & $1035$ & $897$ & $791$ \\
    \hline
  \end{tabular}
\end{table}

\noticka{
kontext jako matice:
\begin{align*}
  I = 
  \left(
    \renewcommand{\arraystretch}{0.6}
    \arraycolsep=0.7mm
    \begin{array}{*{10}{c}}
      0.50 & 1.00 & 1.00 & 1.00 & 0.75 & 1.00 & 0.75 & 0.75 & 1.00 & 0.75 \\
      1.00 & 1.00 & 0.75 & 0.75 & 0.50 & 1.00 & 0.75 & 0.50 & 1.00 & 0.50 \\
      1.00 & 1.00 & 0.75 & 0.75 & 1.00 & 1.00 & 1.00 & 0.25 & 0.25 & 0.75 \\
      0.50 & 0.50 & 0.75 & 1.00 & 0.75 & 0.50 & 0.75 & 0.25 & 0.50 & 1.00 \\
      0.75 & 0.75 & 0.50 & 0.50 & 0.75 & 1.00 & 0.25 & 0.50 & 0.25 & 0.75 \\
    \end{array}
  \right)
\end{align*}
}

Using shades of gray to represent grades from the five-element scale $L$,
the matrix $I$ corresponding to Tab.\,\ref{tab:dec_input}\,(bottom) can be 
visualized in the following array (rows correspond to athletes,
columns correspond to disciplines, the darker the array entry, the
higher the score):
\begin{align*}
  \text{\includegraphics{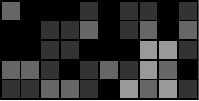}}
\end{align*}

\noticka{
The concept lattice associated to $I$ contains 
celkovy pocet konceptu: $128$
}

The algorithm described in Section \ref{sec:a} found a set $\cal F$
of $7$ formal concepts which factorize $I$, i.e. for which
$I=A_{\cal F}\circ B_{\cal F}$ (note that in this example, we have
used the {\L}ukasiewicz t-norm on $L$).
These factor concepts are shown in Table~\ref{tab:fac_concepts}
in the order in which they were produced by the algorithm.
In addition, Fig.\,\ref{tab:fac_rect} shows the corresponding
rectangular matrices, cf.~\eqref{eqn:rec}.

For example, factor concept $F_1$ applies to Sebrle to degree $0.5$,
to both Clay and Karpov to degree $1$, to Macey to degree $0.5$,
and to Warners to degree $0.75$.
Furthermore, this factor concept applies to attribute
$10$ (100\,m) to degree $1$, to attribute
$lj$ (long jump) to degree $1$, to attribute sp (shot put) 
to degree $0.75$, etc.
This means that an excellent performance (degree $1$) in 100\,m,
an excellent performance in long jump, a very good performance (degree $0.75$)
in shot put, etc. are particular manifestations of this factor concept.
On the other hand, only a relatively weak performance (degree $0.25$) 
in javelin throw and pole vault are manifestations of this factor.

\begin{table*}[t]
  \caption{Factor concepts}
  \label{tab:fac_concepts}
  \centering

  \smallskip
  \small
  \scalebox{.8}{%
  \begin{tabular}{@{}|c|l|l|@{}}
    \hline
    \rule{0pt}{9pt}$F_i$ & \emph{Extent} & \emph{Intent} \\
    \hline
    \hline
    \rule{0pt}{9pt}$F_1$ &
    $\{\deg{.5}{\text{S}},\text{C},\text{K},
    \deg{.5}{\text{M}},\deg{.75}{\text{W}}\}$ &
    $\{10,\text{lj},\deg{.75}{\text{sp}},\deg{.75}{\text{hj}},\deg{.5}{40},11,
    \deg{.5}{\text{di}},\deg{.25}{\text{pv}},\deg{.25}{\text{ja}},
    \deg{.5}{15}\}$ \\
    \rule{0pt}{9pt}$F_2$ &
    $\{\text{S},\deg{.75}{\text{C}},\deg{.25}{\text{K}},
    \deg{.5}{\text{M}},\deg{.25}{\text{W}}\}$ &
    $\{\deg{.5}{10},\text{lj},\text{sp},\text{hj},\deg{.75}{40},11,
    \deg{.75}{\text{di}},\deg{.75}{\text{pv}},\text{ja},
    \deg{.75}{15}\}$ \\
    \rule{0pt}{9pt}$F_3$ &
    $\{\deg{.75}{\text{S}},\deg{.5}{\text{C}},
    \deg{.75}{\text{K}},\text{M},\deg{.5}{\text{W}}\}$ &
    $\{\deg{.5}{10},\deg{.5}{\text{lj}},\deg{.75}{\text{sp}},\text{hj},
    \deg{.75}{40},\deg{.5}{11},\deg{.75}{\text{di}},\deg{.25}{\text{pv}},
    \deg{.5}{\text{ja}},15\}$ \\
    \rule{0pt}{9pt}$F_4$ &
    $\{\text{S},\deg{.75}{\text{C}},\deg{.75}{\text{K}},
    \deg{.5}{\text{M}},\text{W}\}$ &
    $\{\deg{.5}{10},\deg{.75}{\text{lj}},\deg{.5}{\text{sp}},
    \deg{.5}{\text{hj}},\deg{.75}{40},11,\deg{.25}{\text{di}},
    \deg{.5}{\text{pv}},\deg{.25}{\text{ja}},\deg{.75}{15}\}$ \\
    \rule{0pt}{9pt}$F_5$ &
    $\{\deg{.75}{\text{S}},\deg{.75}{\text{C}},\text{K},
    \deg{.75}{\text{M}},\deg{.25}{\text{W}}\}$ &
    $\{\deg{.75}{10},\deg{.75}{\text{lj}},\deg{.75}{\text{sp}},
    \deg{.75}{\text{hj}},\deg{.75}{40},\deg{.75}{11},\text{di},
    \deg{.25}{\text{pv}},\deg{.25}{\text{ja}},\deg{.75}{15}\}$ \\
    \rule{0pt}{9pt}$F_6$ &
    $\{\deg{.75}{\text{S}},\deg{.5}{\text{C}},\text{K},
    \deg{.75}{\text{M}},\deg{.75}{\text{W}}\}$ &
    $\{\deg{.75}{10},\deg{.75}{\text{lj}},\deg{.75}{\text{sp}},
    \deg{.75}{\text{hj}},40,\deg{.75}{11},\deg{.5}{\text{di}},
    \deg{.25}{\text{pv}},\deg{.25}{\text{ja}},\deg{.75}{15}\}$ \\
    \rule{0pt}{9pt}$F_7$ &
    $\{\text{S},\text{C},\deg{.25}{\text{K}},
    \deg{.5}{\text{M}},\deg{.25}{\text{W}}\}$ &
    $\{\deg{.5}{10},\text{lj},\deg{.75}{\text{sp}},\deg{.75}{\text{hj}},
    \deg{.5}{40},11,\deg{.75}{\text{di}},\deg{.5}{\text{pv}},\text{ja},
    \deg{.5}{15}\}\rangle$ \\
    \hline
  \end{tabular}}
\end{table*}

Therefore, a decomposition $I = A_\mathcal{F} \circ B_\mathcal{F}$
exists with $7$ factors
where:
\begin{align*}
  A_\mathcal{F} &=
  \left(
    \renewcommand{\arraystretch}{0.6}
    \arraycolsep=0.7mm
    \begin{array}{*{7}{c}}
      0.50 & 1.00 & 0.75 & 1.00 & 0.75 & 0.75 & 1.00 \\
      1.00 & 0.75 & 0.50 & 0.75 & 0.75 & 0.50 & 1.00 \\
      1.00 & 0.25 & 0.75 & 0.75 & 1.00 & 1.00 & 0.25 \\
      0.50 & 0.50 & 1.00 & 0.50 & 0.75 & 0.75 & 0.50 \\
      0.75 & 0.25 & 0.50 & 1.00 & 0.25 & 0.75 & 0.25 \\
    \end{array}
  \right)\!, \\[4pt]
  B_\mathcal{F} &=
  \left(
    \renewcommand{\arraystretch}{0.6}
    \arraycolsep=0.7mm
    \begin{array}{*{10}{c}}
      1.00 & 1.00 & 0.75 & 0.75 & 0.50 & 1.00 & 0.50 & 0.25 & 0.25 & 0.50 \\
      0.50 & 1.00 & 1.00 & 1.00 & 0.75 & 1.00 & 0.75 & 0.75 & 1.00 & 0.75 \\
      0.50 & 0.50 & 0.75 & 1.00 & 0.75 & 0.50 & 0.75 & 0.25 & 0.50 & 1.00 \\
      0.50 & 0.75 & 0.50 & 0.50 & 0.75 & 1.00 & 0.25 & 0.50 & 0.25 & 0.75 \\
      0.75 & 0.75 & 0.75 & 0.75 & 0.75 & 0.75 & 1.00 & 0.25 & 0.25 & 0.75 \\
      0.75 & 0.75 & 0.75 & 0.75 & 1.00 & 0.75 & 0.50 & 0.25 & 0.25 & 0.75 \\
      0.50 & 1.00 & 0.75 & 0.75 & 0.50 & 1.00 & 0.75 & 0.50 & 1.00 & 0.50 \\
    \end{array}
  \right)\!.
\end{align*}
Again, using shades of gray, this decomposition can be depicted as:
\begin{align*}
  \text{\lower-2mm\hbox{\includegraphics{figure01.pdf}}}
  \mathop{\lower-6mm\hbox{$=$}}
  \text{\lower-2mm\hbox{\includegraphics{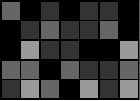}}}
  \mathop{\lower-6mm\hbox{$\circ$}}
  \text{\includegraphics{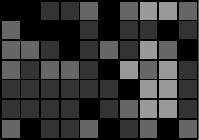}}
\end{align*}

\begin{figure}
  \centering
  \begin{tabular}{@{}c@{~}c@{~}c@{}}
    \includegraphics{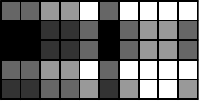} &
    \includegraphics{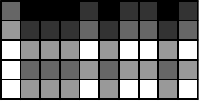} &
    \includegraphics{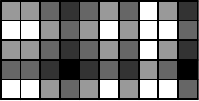} \\[-1pt]
    $F_1$ & $F_2$ & $F_3$
  \end{tabular} \\[1ex]
  \begin{tabular}{@{}c@{~}c@{~}c@{~}c@{}}
    \includegraphics{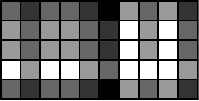} &
    \includegraphics{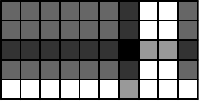} &
    \includegraphics{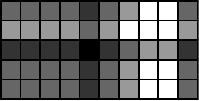} &
    \includegraphics{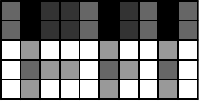} \\[-1pt]
    $F_4$ & $F_5$ & $F_6$ & $F_7$ \\[-6pt]
  \end{tabular}
  \caption{Factor concepts as rectangular patterns}
  \label{tab:fac_rect}
\end{figure}

Fig.\,\ref{tab:fac_fill} demonstrates what portion of the data 
matrix $I$ is explained using just some of the factor concepts from 
$\cal F$. 
The first matrix labeled by $46\%$ shows 
$A_{{\cal F}_1}\circ B_{{\cal F}_1}$ for ${\cal F}_1$ consisting
of the first factor $F_1$ only. That is, the matrix is just
the rectangular pattern corresponding to $F_1$, cf. 
Fig.\,\ref{tab:fac_rect}. 
As we can see, this matrix is contained in $I$,
i.e. approximates $I$ from below, in that
$(A_{{\cal F}_1}\circ B_{{\cal F}_1})_{ij}\leq I_{ij}$ for all
entries (row $i$, column $j$). 
Label $46\%$ indicates that $46\%$ of the entries
of $A_{{\cal F}_1}\circ B_{{\cal F}_1}$ and $I$ are equal. 
In this sense, the first factor explains $46\%$ of the data.
Note however, that several of the $54\%=100\%-46\%$ of the other entries 
of $A_{{\cal F}_1}\circ B_{{\cal F}_1}$ are close to the corresponding entries of $I$,
so a measure of closeness of $A_{{\cal F}_1}\circ B_{{\cal F}_1}$ and $I$
which takes into account also close entries, rather than exactly equal ones only,
would yield a number larger than $46\%$.

The second matrix in Fig.\,\ref{tab:fac_fill}, with label $72\%$,
shows $A_{{\cal F}_2}\circ B_{{\cal F}_2}$ for ${\cal F}_2$ consisting
of $F_1$ and $F_2$. That is, the matrix demonstrates what portion
of the data matrix $I$ is explained by the first two factors.
Again, $A_{{\cal F}_2}\circ B_{{\cal F}_2}$ approximates $I$ from below
and $72\%$ of the entries of $A_{{\cal F}_2}\circ B_{{\cal F}_2}$ and $I$
coincide now.
Note again that even for the remaining $28\%$ of entries,
$A_{{\cal F}_2}\circ B_{{\cal F}_2}$ provides a reasonable
approximation of $I$, as can be seen by comparing the matrices
representing $A_{{\cal F}_2}\circ B_{{\cal F}_2}$ and $I$, i.e.
the one labeled by $72\%$ and the one labelled by $100\%$.

Similarly, the matrices labeled by $84\%$, $92\%$, $96\%$, 
$98\%$, and $100\%$ represent 
$A_{{\cal F}_l}\circ B_{{\cal F}_l}$ for $l=3,4,5,6,7$,
for sets ${\cal F}_l$ of factor concepts consisting
of $F_1,\dots,F_l$. 
We can conclude from the visual inspection of the matrices that
already the two or three first factors explain the data
reasonably well.

\begin{figure}
  \centering
  \begin{tabular}{@{}c@{~}c@{~}c@{}}
    \includegraphics{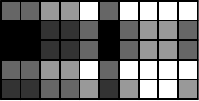} &
    \includegraphics{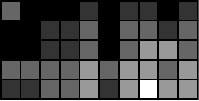} &
    \includegraphics{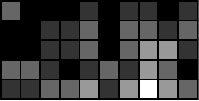} \\[-1pt]
    $46\%$ & $72\%$ & $84\%$
  \end{tabular} \\[1ex]
  \begin{tabular}{@{}c@{~}c@{~}c@{~}c@{}}
    \includegraphics{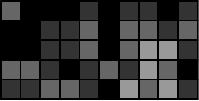} &
    \includegraphics{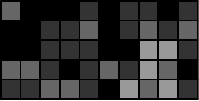} &
    \includegraphics{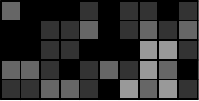} &
    \includegraphics{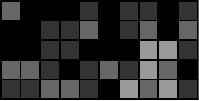} \\[-1pt]
    $92\%$ & $96\%$ & $98\%$ & $100\%$ \\[-6pt]
  \end{tabular}
  \caption{$\boldsymbol{\bigvee}$-superposition of factor concepts}
  \label{tab:fac_fill}
\end{figure}

Let us now focus on the interpretation of the factors.
Fig.\,\ref{tab:fac_rect} is helpful as it shows the 
clusters corresponding to the factor concepts which draw
together the athletes and their performances in the events.

Factor $F_1$: Manifestations of this factor with grade $1$ are
$100$\,m, long jump, $110$\,m hurdles.
This factor can be interpreted as the ability to run fast
for short distances
 (speed).
Note that this factor applies particularly to Clay and Karpov
which is well known in the world of decathlon.
Factor $F_2$: Manifestations of this factor with grade $1$ are
long jump, shot put, high jump, $110$\,m hurdles,
javelin.
$F_2$ can be interpreted as the ability to apply very high force
in a very short term (explosiveness).
$F_2$ applies particularly to Sebrle, and then to Clay, who
are known for this ability. 
Factor $F_3$: Manifestations with grade $1$ are high jump and 
$1500$\,m. This factor is typical for lighter, not 
very muscular athletes (too much muscles prevent jumping high
and running long distances).
Macey, who is evidently that type among decathletes 
($196$\,cm and $98$\,kg) is the athlete to whom the factor applies
to degree $1$.
These are the most important factors behind data matrix $I$.

\noticka{
JAK TO DOPADLO PRO CELE STARTOVNI POLE:
tabulka ma $28$ radku (zavodniku), pocet formalnich konceptu je $11175$;
faktory vypadaji tak, ze pro exaktni faktorizaci bylo nalezeno $14$ faktoru,
tohle nevim jestli rikat nebo jak to podat. Kazdopadne pro prvnich sest
faktoru pokryva 25.8\%, 41.3\%, 52.8\%, 62.3\%, 69.8\% a 77\% zaznamu.
}


\subsection{Experimental Evaluation}\label{sec:ee}
We now present experiments with exact and approximate factorization of
selected publicly-available datasets and randomly generated matrices and
their evaluation.
First, we observed how close is the number of
factors found by the algorithm \textsc{FindFactors} to a known number
of factors in artificially created matrices. In this experiment, we were generating $20 \times 20$ matrices
according to various distributions of $5$ grades. These matrices were
generated by multiplying $m \times k$ and $k \times n$ matrices. Therefore,
the resulting matrices were factorizable with at most $k$ factors. 
Then, we executed
the algorithm to find $\cal F$ and observed how close is the number
$|{\cal F}|$ of factors to $k$. 
The results are depicted in Tab.\,\ref{tab:exact_fac}.
We have observed that in the average case, the choice of a t-norm is not
essential and all t-norms give approximately the same results.
In particular, Tab.\,\ref{tab:exact_fac} describes results
for \L ukasiewicz and minimum t-norms.
Rows of Tab.\,\ref{tab:exact_fac} correspond
to numbers 
$k=5,7,\dots,15$ denoting the known number of factors.
For each $k$, we computed the average number of factors
produced by our algorithm in $2000$
$k$-factorizable matrices. The average values are written in the
form of ``average number of factors $\pm$ standard deviation''. 

\begin{table}
  \caption{Exact factorizability}
  \label{tab:exact_fac}
  \centering
  \small

  \medskip
  \begin{tabular}{@{\,}r|r@{$\,\pm\,$}lr@{$\,\pm\,$}l@{\,}}
    \toprule
    \multicolumn{1}{@{\,}c@{~}}{} &
    \multicolumn{2}{@{~}c@{~}}{\L ukasiewicz $\otimes$} &
    \multicolumn{2}{@{~}c@{\,}}{minimum $\otimes$} \\
    \multicolumn{1}{@{\,}c@{~}}{$k$} &
    \multicolumn{2}{@{~}c@{~}}{no. of factors} &
    \multicolumn{2}{@{~}c@{\,}}{no. of factors} \\
    \midrule
    $5$ & $5.205$ & $0.460$ & $6.202$ & $1.037$ \\
    $7$ & $7.717$ & $0.878$ & $10.050$ & $1.444$ \\
    $9$ & $10.644$ & $1.316$ & $13.379$ & $1.676$ \\
    $11$ & $13.640$ & $1.615$ & $15.698$ & $1.753$ \\
    $13$ & $16.423$ & $1.879$ & $17.477$ & $1.787$ \\
    $15$ & $18.601$ & $2.016$ & $18.721$ & $1.863$ \\
    \bottomrule
  \end{tabular}
\end{table}

As mentioned above, factorization and factor analysis of binary data is
a special case of our setting with $L=\{0,1\}$, 
i.e. with the scale containing just two grades.
Then, the matrix product $\circ$ given by (\ref{eqn:dec}) coincides
with the Boolean matrix multiplication and
the problem of decomposition of graded matrices 
coincides with the problem of decomposition of binary matrices
into the Boolean product of binary matrices.
We performed experiments with our algorithm in this particular
case with three large binary data sets (binary matrices)
from the Frequent Itemset Mining Dataset
Repository\footnote{\url{http://fimi.cs.helsinki.fi/data/}}.
In particular, we considered 
the CHESS ($3196 \times 75$ binary matrix),
CONNECT ($67557 \times 129$ binary matrix),
and MUSHROOM ($8124 \times 119$ binary matrix) data sets.
The results are shown in Fig.\,\ref{fig:bool}.
The $x$-axes correspond to the number of factors
(from $1$ up to $50$ factors were observed) and
the $y$-axes are percentages of data explained by the factors.
For example, we can see that the first $10$ factors of CHESS explain more
than $70\%$ of the data, i.e. $A_{\cal F}\circ B_{\cal F}$ covers more than
$70\%$ of the nonzero entries of CHESS for $|{\cal F}|=10$. In all the three
cases, we can see a tendency that a relatively small number of factors
(compared to the number of attributes in the datasets) cover a significant
part of the data.

\begin{figure}
  \centering
  \begin{tabular}{@{}c@{~~}c@{~~}c@{}}
    \includegraphics{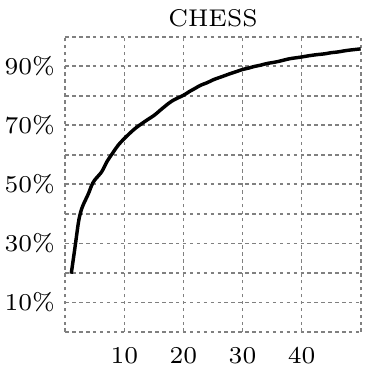} &
    \includegraphics{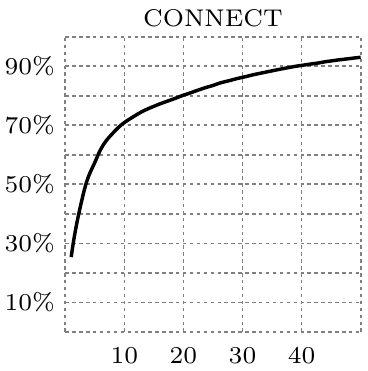} &
    \includegraphics{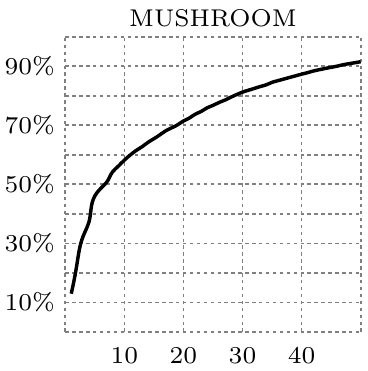}
  \end{tabular}

  \caption{Approximate factorization of Boolean matrices by first 50 factors}
  \label{fig:bool}
\end{figure}

A similar tendency can also be observed for graded incidence data. For instance,
we have utilized the algorithm in factor analysis of
the FOREST FIRES~\cite{CoMo:Forest} dataset from
the UCI Machine Learning Repository\footnote{\url{http://archive.ics.uci.edu/ml/}}.
In its original form, the dataset contains real values. It has been therefore
transformed into a graded incidence matrix representing relationship between
spatial coordinates within the Montesinho park map (rows) and 50 different groups
of environmental and climate conditions (columns). The matrix entries are degrees
(coming from an equidistant \L ukasiewicz chain
$L = \{\frac{n}{100} \,|\, n \text{ is integer, } 0 \leq n \leq 100\}$)
to which there has been a large area of burnt forest in the sector of the map under
the environmental conditions. Factor analysis of data in this form can help reveal
factors which contribute to forests burns in the park. The exact factorization
has revealed 46 factors which explain 50 attributes. As in case of the Boolean
datasets, relatively small number of factors explain large portions of the data.
For instance, more than $50\%$ of the data is
covered by $10$ factors, more than $80\%$ of the data is covered by $23$ factors,
see~Fig.\,\ref{fig:forest}.

\begin{figure}
  \centering
  \includegraphics{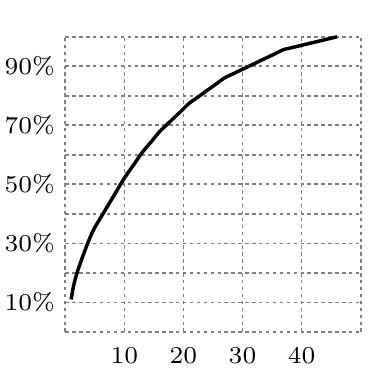}

  \caption{Factorization of graded incidence matrix FOREST FIRES}
  \label{fig:forest}
\end{figure}

\section{Conclusions}
\label{sec:c}

We presented a novel approach to decomposition and factor analysis of matrices with
grades, i.e. of a particular form of ordinal data. The factors in this approach 
correspond to formal concepts in the data matrix.
The approach is justified by a theorem according to which
optimal decompositions are attained by using formal concepts as factors.
The relationship between the factors
and original attributes is a non-linear one.
An advantageous feature of the model is a transparent way
of treating the grades which results in good interpretability of factors.
We observed that the decomposition problem is NP-hard as an optimization
problem.
We proposed a greedy  algorithm for computing suboptimal decompositions
and provided results of experiments demonstrating its behavior.
Furthermore, we presented a detailed example of factor discovery
which demonstrates that the method yields interesting factors from data.
Since the method developed naturally allows for a linguistic interpretation
of factors, it may be considered as a step toward what might be regarded  
a linguistic factor analysis of qualitative data.

Future research will include the following topics.
First, a comparison, both theoretical and experimental, to other methods
  of matrix decompositions,  in particular to the methods emphasizing 
  good interpretability, such as non-negative matrix factorization
  \cite{LeSe:Lponmf}.
 Second, an investigation of  approximate decompositions
  of $I$, i.e. decompositions to $A$ and $B$ for which $A\circ B$
is approximately equal to $I$ with respect to a reasonable notion
of approximate equality.
Third, development of further theoretical insight focusing particularly
  on reducing further the space of factors to which the search
  for factors can be restricted. 
Fourth,  study the computational complexity aspects of the problem of approximate
factorization, in particular the approximability of the problem of finding
decompositions of matrix $I$ \cite{Auea:CA}.
Fifth, explore further the applications of the decompositions studied in this paper,
particularly in areas such as psychology, sports data, or customer surveys, where
ordinal data is abundant.

\section*{Acknowledgment}
\noindent
 R. Belohlavek acknowledges
supported by grant No. P202/10/0262 of the Czech Science Foundation. 
V. Vychodil acknowledges support by the ESF project No. CZ.1.07/2.3.00/20.0059, the project
is co-financed by the European Social Fund and the state budget of the Czech Republic
    Preliminary version of this paper was presented at
    the International Conference on Formal Concept Analysis, Darmdstadt, Germany,
    in 2009.


\end{document}